\documentclass[preprint,12pt]{elsarticle}

\usepackage{subfig}
\usepackage{amssymb,amsmath,amsthm,amsfonts}
\newtheorem{remark}{Remark}
\newtheorem{theorem}{Theorem}
\newtheorem{lemma}{Lemma}
\newtheorem{assumption}{Assumption}
\usepackage{algorithmic}
\usepackage{algorithm}

\journal{Neural Networks}

\begin{document}

\begin{frontmatter}



\title{Beyond $\ell_2$-norm and $\ell_\infty$-norm: A Curvature-Inspired $\ell_p$-Norm Scheme for Deep Neural Networks}






\author[rvt] {Jianhao Xu}
\ead{jianhaoxu2025@163.com}
\author[rvt]{Zhuang Yang\corref{cor1}}
\ead{zhuangyng@163.com}
\cortext[cor1]{Corresponding author}
\address[rvt]{School of Computer Science and Technology, Soochow University, Suzhou, 215006, China}

\begin{abstract}
The existing optimizers for deep neural networks (DNNs) typically rely on either the $\ell_2$ norm or the $\ell_\infty$ norm, resulting in optimizers that do not adapt well to substantial changes in curvature across parameter dimensions. Generally, the training process of DNNs often exhibits strong curvature anisotropy in the early period, whereas in the later period, the training process of DNNs tends to move toward flatter regions with weaker anisotropy. Particularly, optimizers based on the \(\ell_2\)-norm are usually dominated by high-curvature directions, restricting updates of optimizers along with lower curvature direction and thus leading to a slower convergence rate. While optimizers based on the \(\ell_\infty\)-norm are prone to oscillations in flatter regions, due to the coordinate-wise updates of the same magnitude. To address these two extreme cases generated by $\ell_2$ and $\ell_\infty$ norms, we propose a novel $\ell_p$-norm scheme with a dynamical value of $p$ and incorporate it into stochastic gradient descent (SGD) and SGD with momentum (SGDM), leading to two novel optimizers with better generalization performance: ${\ell_p}$-SGD (LPSGD) and ${\ell_p}$-SGDM (LPSGDM). Particularly, the resulting optimizers suppress the dominance of high-curvature directions in the early period by utilizing a large $p$ ($p>2$), followed by a gradual decrease of $p$ toward 2 to enable more stable and refined updates, where the latter process is motivated by the cosine annealing strategy. We establish theoretical guarantees of the resulting algorithms and analyze that both LPSGD and LPSGDM achieve an \(O(T^{-1/2})\) convergence rate for the nonconvex setting. Extensive experiments are conducted on benchmark datasets, including CIFAR-10, CIFAR-100, and ImageNet-1K, with multiple DNNs such as VGG-11, ResNet-18, and ResNet-50. Demonstrably, all the experimental results effectively verify the superiority of the proposed algorithms.
\end{abstract}




\end{frontmatter}



\section{Introduction}
The training of deep neural networks (DNNs) can be formulated as the following empirical risk minimization problem:
\begin{equation}
\label{eq:optimization}
\min_{\theta \in R^d} F(\theta)=\frac{1}{n}\sum_{i=1}^n f_i(\theta),
\end{equation}
where \(\theta \in R^d\) denotes the network parameters, \(n\) is the number of training samples, and \(f_i(\theta)\) is the loss on the \(i\)-th sample. Although Problem \eqref{eq:optimization} looks very simple,  optimization problems closely related to Problem \eqref{eq:optimization} also arise in tasks such as federated learning \cite{li2023}, reinforcement learning \cite{Noorani2025} and large language models \cite{wang}. Practically, Problem \eqref{eq:optimization} is highly nonlinear and nonconvex, bringing out a significant challenge to current optimizers.

For Problem \eqref{eq:optimization}, stochastic gradient descent (SGD) \cite{sgd} is broadly used with the following  simple iteration scheme:
\begin{equation}
\label{eq:sgd}
\theta_{k+1} = \theta_k - \eta \nabla f_i(\theta_k),
\end{equation}
where \(\eta>0\) is the learning rate, and \(\nabla f_i(\theta_k)\) denotes the gradient of the loss function $f_i$ at point $\theta_k$. The iterative scheme \eqref{eq:sgd} is naturally rooted in Euclidean (\(\ell_2\)-norm) geometry, but has achieved significant success in large-scale model training, such as convolutional neural networks \cite{han} and transformer-based language models \cite{porian}. Beyond vanilla SGD, a variety of its variants have been developed, including adaptive gradient methods such as Adam \cite{adam}, AdamW \cite{adamw},  AdaGrad \cite{adagrad} and RMSProp \cite{rmsprop}.
In contrast, sign-based methods are naturally rooted in $\ell_\infty$ geometry \cite{Balles}, where the update of sign-based methods depends only on the coordinate-wise signs of the gradients and particularly iterates as follows:
\begin{equation}
\label{eq:signsgd}
\theta_{k+1} = \theta_k - \eta  \text{sign}(\nabla f_{i}(\theta_k)),
\end{equation}
representative examples include signSGD \cite{signsgd} and more recent sign-driven optimizers such as Lion \cite{lion} and Lion-$K$ \cite{lionk}. More specifically, Adam can also be viewed as a variant of signSGD under the case $\beta_1=\beta_2=0$ \cite{Balles,adamsignsgd}.

Under the \(\ell_2\)-norm geometry, Dauphin et al. \cite{Dauphin} pointed out that highly anisotropic curvature makes gradient descent be dominated by high-curvature directions, resulting in the requirement of a small stable learning rate (or equivalently, the step size), which slows the training progress of the model along flat directions.  Cohen et al. \cite{Cohen} further proved that the relevant stability threshold of the learning rate is determined by the largest Hessian eigenvalue, particularly with the learning rate scale on the order of \(1/\lambda_{\max}\). While under the \(\ell_\infty\)-norm geometry, Balles et al. \cite{Balles} pointed out that sign-based updates are theoretically favored when the Hessian is close to diagonal and highly anisotropic. Yet, Li et al. \cite{Li} proved that signGD is not generally convergent even on smooth strongly convex problems. Summarily, the above analyses indicate that the suitability of a given norm geometry depends strongly on the curvature structure, specifically, on the dispersion of the Hessian eigenvalues.

Furthermore, recent studies have shown that the loss function of DNNs usually exhibits different curvatures across different parameter dimensions. 
Ghorbani et al.~\cite{anisotropy1} analyzed Hessian eigenvalue densities in neural net optimization and found a spectra consisting of a dense bulk near zero together with a few large outliers. 
To better illustrate the above discussion, we plot the evolution of the Hessian eigenvalue distribution during the training of ResNet-18 on CIFAR-10 and ResNet-50 on ImageNet-1K. In Fig. 1, the horizontal axis denotes eigenvalues, and the vertical axis denotes training epochs. As illustrated in Fig. 1, at the beginning of training, the set of Hessian eigenvalues contains clear outlier large eigenvalues (e.g., 17.5, 1180) along with some negative eigenvalues (e.g., -2.5, -680), indicating highly curvature anisotropic and nonconvexity. As training progresses, most eigenvalues fluctuate around zero, the occurrence of extreme eigenvalues becomes less frequent, and negative eigenvalues disappear, suggesting that the model transitions to a weak curvature, approximately convex region.

\begin{figure}[!t]
\centering
\subfloat[\label{fig_1_a}]{
    \includegraphics[width=0.7\textwidth]{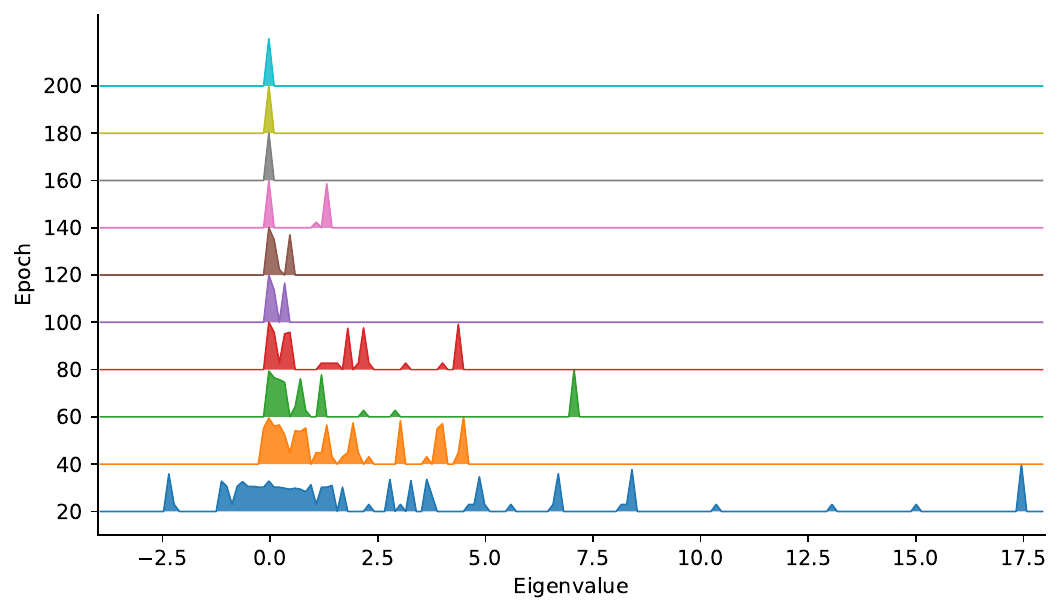}
}
\hfill
\subfloat[\label{fig_1_b}]{
    \includegraphics[width=0.7\textwidth]{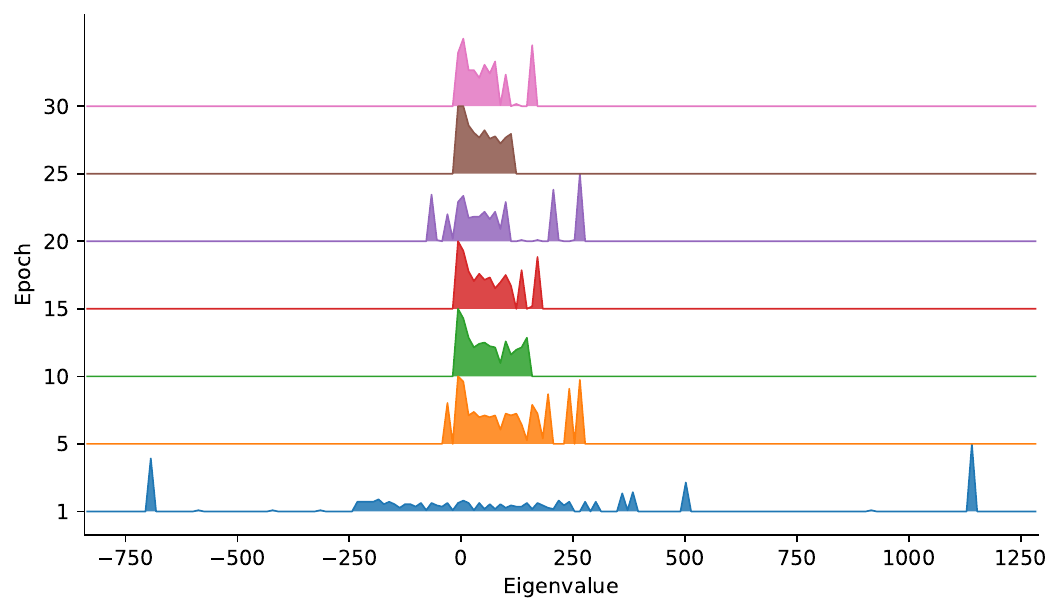}
}
\caption{The evolution of Hessian eigenvalue distribution during training. (a) ResNet-18 on CIFAR-10, sampled every 20 epochs. (b) ResNet-50 on ImageNet-1K, sampled every 5 epochs.}
\label{fig_1}
\end{figure}

As shown in Fig.~1, the Hessian eigenvalue distribution evolves from being widely dispersed with pronounced outliers to becoming more concentrated over training. Motivated by this observation, together with recent studies on optimization under the \(\ell_2\)-norm and \(\ell_\infty\)-norm geometries, we formulate the following key research questions:

\begin{itemize}
\item[] \textbf{Question 1:} \emph{Do we need to fasten the norm scheme, i.e., only using $\ell_2$-norm, or $\ell_\infty$-norm, separately, in the training of DNNs, facing the challenge of substantial change in curvature structure?}

\item[] \textbf{Question 2:} \emph{If not, do there exist some certain types of the norm scheme (except for $\ell_2$-norm and $\ell_\infty$-norm) to appropriately match the curvature structure at different training stages in the training of DNNs?}
\end{itemize}

We provide affirmative answers for the above two questions. We introduce a cosine $\ell_p$-norm scheme that challenges the conventional approaches of employing a fixed norm geometry. Notably, this scheme captures curvature variation with negligible computational overhead and can be naturally extended to existing first-order optimizers. Empirical evidence shows that, when applied to SGD and SGDM, the proposed strategy consistently enhances the generalization of SGD and SGDM.

\begin{figure*}[!t]
\centering
\subfloat[]{\includegraphics[width=3.5in]{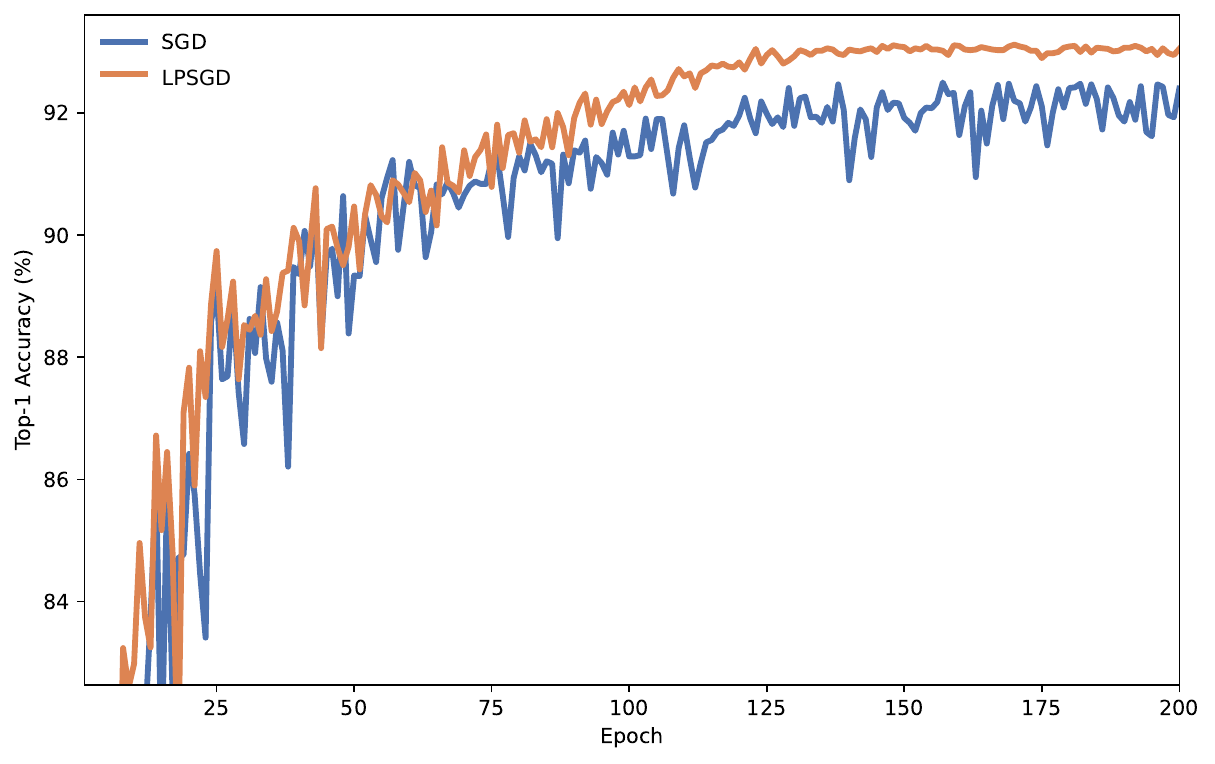}%
\label{fig_first_case}}
\hfil
\subfloat[]{\includegraphics[width=3.5in]{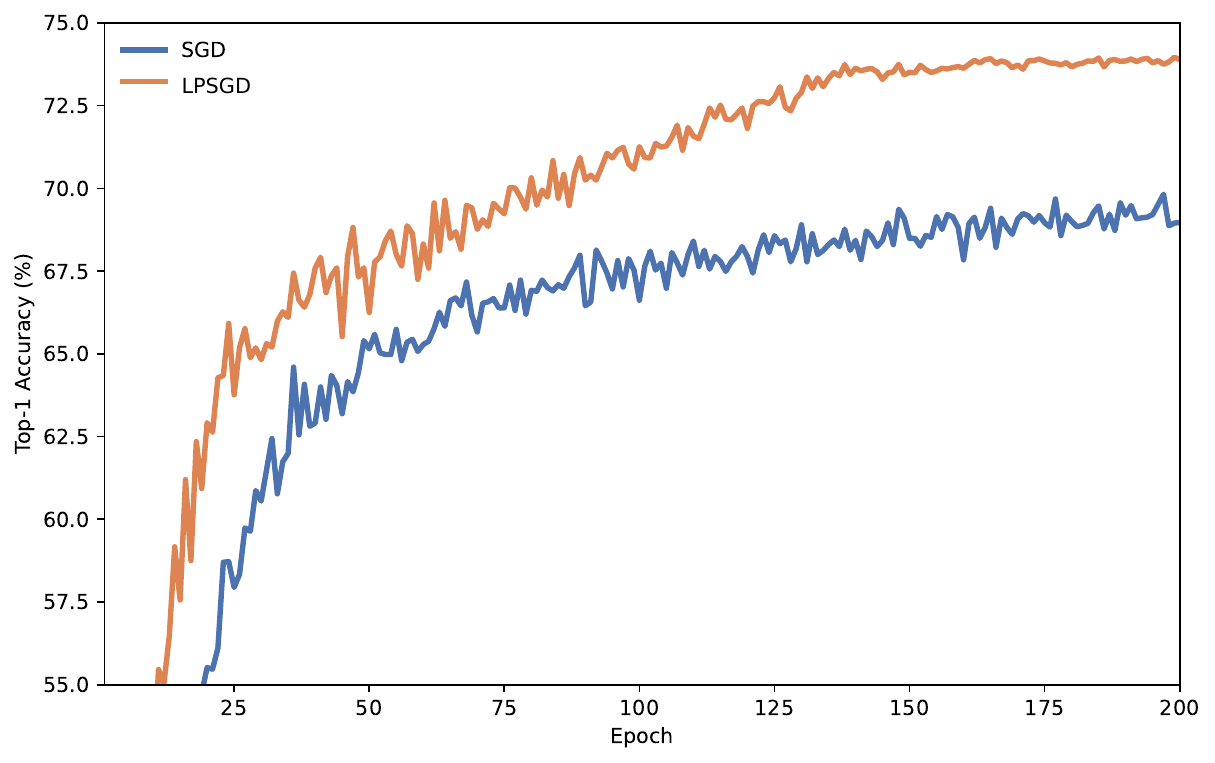}%
\label{fig_second_case}}
\caption{(a) Testing top-1 accuracy of SGD and LPSGD on ResNet-18/CIFAR-10. (b) Testing top-1 accuracy of SGD and LPSGD on ResNet-50/CIFAR-100.}
\label{fig_sim}
\end{figure*}

The main contributions of this work are summarized as follows:
\begin{itemize}
    \item We identify and characterize the evolution of curvature anisotropy during DNNs training. Specifically, we eliminate the diverse effect of such curvature anisotropy on practical applications by developing a cosine $\ell_p$-norm scheme.
    \item We apply cosine $\ell_p$-norm scheme to SGD and SGD with momentum (SGDM), leading to LPSGD and LPSGDM. More specifically, we rigorously prove the convergence behavior of the proposed algorithms for nonconvex optimization problems and show that they achieve a convergence rate of \(O(T^{-1/2})\).
    \item Extensive experiments are conducted on benchmark datasets, including CIFAR-10, CIFAR-100, and ImageNet-1K, using multiple DNN architectures, including VGG-11, ResNet-18, and ResNet-50. Specifically, the proposed algorithms outperform state-of-the-art baselines by 0.68\%, 0.31\%, and 0.27\% on CIFAR-10, and by 1.12\%, 1.11\%, and 1.87\% on CIFAR-100, respectively. On ImageNet-1K with ResNet-50, they further achieve a 1.89\% improvement over the best baseline.
\end{itemize}

To provide an initial empirical validation for the cosine \(\ell_p\)-norm scheme in first-order methods, we first compare SGD and the resulting algorithm, LPSGD (Algorithm \ref{LPSGD}), where the former is a significantly popular optimizer in DNN. Specifically, we conduct two simple experiments on ResNet-18\cite{resnet}/CIFAR-10\cite{cifar} and ResNet-50\cite{resnet}/CIFAR-100\cite{cifar}. The results are reported in Fig.~2, where the horizontal axis denotes the training epoch and the vertical axis denotes testing top-1 accuracy. Notice that, LPSGD denotes SGD augmented with the proposed cosine $\ell_p$-norm scheme. Fig.~2 demonstrates that in contrast to SGD, LPSGD improves the final testing top-1 accuracy from \(92.40\%\) to \(93.05\%\) on ResNet-18/CIFAR-10, and from \(68.97\%\) to \(73.89\%\) on ResNet-50/CIFAR-100, providing preliminary evidence for the benefit of cosine \(\ell_p\)-norm scheme in improving generalization.

\section{Related Work}
In order to follow us easily, this work will discuss the prior studies from two perspectives in this section. Specifically, we first review the methods for evaluating curvature information in DNNs, then discuss those optimizers that explicitly exploit curvature information.

{\bf{Curvature Estimation in Deep Models:}} Based on the Hessian-based analysis technique,
several studies have studied how to characterize and utilize the curvature information of DNNs. Early empirical studies by Sagun et al. \cite{sagun} revealed that the Hessian spectrum of over-parameterized deep models typically consists of a bulk concentrated near zero together with a small number of isolated outlier eigenvalues, highlighting strong spectral anisotropy. Pennington and Bahri \cite{pennington} further analyzed the Hessian spectra from the perspective of random matrix theory, providing a theoretical view of the eigenvalue distribution of neural network loss surfaces. To make such analysis scalable, Ghorbani et al. \cite{ghorbani} developed a stochastic Lanczos quadrature-based method for estimating the spectral density of the Hessian matrix during DNN training. Yao et al. \cite{yao} subsequently introduced PyHessian, a practical framework for computing top Hessian eigenvalues, Hessian trace, and Hessian spectral density in large-scale deep models. More recently, Sankar et al. \cite{sankar} extended Hessian spectral analysis to the layer level, showing that layerwise eigenspectra and trace evolution provide additional insight into the geometry of deep-network training. These studies developed effective methods for measuring curvature information in DNNs and further described the curvature structures observed during deep model training, motivating us to develop more efficient optimizers according to geometry information for deep models.

{\bf{Optimizers Exploiting Curvature Information:}}
Many studies have directly exploited curvature-related or second-order information (mentioned above) in optimization, thus leading to various novel curvature-based optimizers. K-FAC \cite{kfac} approximated natural gradient using Kronecker-factored curvature blocks. Shampoo \cite{shampoo} performed tensor-structured preconditioning through per-dimension second-moment statistics. Sophia \cite{sophia} used a lightweight stochastic approximation to the diagonal Hessian to improve the scalability of second-order optimization. More recently, SALA \cite{sala} combined quadratic trajectory approximation and sharpness-aware updates to seek flatter regions. In contrast, our method only uses first-order information that leverages non-Euclidean geometry, but achieves better performance than state-of-the-art optimizers.

\section{Preliminaries and Assumptions}

Several notations and standard assumptions, used in this work, will be given below.

\textbf{Notations.} Throughout this paper, \(\|\cdot\|_p\) denotes the standard \(\ell_p\)-norm and \(\|\cdot\|_q\) denotes its dual norm, where \(1/p+1/q=1\). For a vector \(x\in \mathbb{R}^d\), \(x^{(i)}\) denotes its \(i\)-th coordinate. We use \(|x|\) and \(\operatorname{sign}(x)\) to denote element-wise absolute value and sign, respectively. In addition, let \(g_t\) denote the stochastic gradient at iteration \(t\), and let \(\nabla f(\theta_t)\) denote the full gradient at \(\theta_t\). 

We make the following assumptions throughout the paper.

\begin{assumption}[\(\ell_{p}\)-smoothness]
\label{ass:lpsmooth}
The objective function \(f:\mathbb{R}^d\to\mathbb{R}\) is L-smooth. Moreover, there exists a constant \(L>0\), then, for all $x,y\in\mathbb{R}^d$,
\begin{equation}
    f(y)\le f(x)+\langle \nabla f(x),y-x\rangle+\frac{L}{2}\|y-x\|_{p}^2.
\end{equation}
\end{assumption}

\begin{assumption}[Unbiased stochastic gradient]
\label{ass:unbiased}
Let \(g_t\) be the stochastic gradient estimator for each iteration \(t\). Then
\begin{equation}
    \mathbb{E}[g_t \mid \theta_t]=\nabla f(\theta_t).
\end{equation}
\end{assumption}

\begin{assumption}[Bounded conditional gradient noise]
\label{ass:noise}
There exists a constant \(\sigma>0\) such that, for any iteration \(t\), the conditional second moment of the gradient noise is bounded as
\begin{equation}
    \mathbb{E}\!\left[\|g_t-\nabla f(\theta_t)\|_{q}^2 \mid \theta_t\right]\le \sigma^2.
\end{equation}
\end{assumption}

\begin{assumption}[Lower bounded objective]
\label{ass:lower}
The objective function is bounded from below, i.e.,
\[
f^\star := \inf_{\theta\in\mathbb{R}^d} f(\theta) > -\infty.
\]
\end{assumption}

\section{Method}

In this section, we will present our cosine \(\ell_p\)-norm scheme for the training of DNNs. Concretely, we first analyze its motivation from the perspective of curvature anisotropy in subsection \ref{method-1}. We then introduce the cosine \(\ell_p\)-norm scheme strategy and show how it can be incorporated into SGD and SGDM, resulting in two practical optimization algorithms, LPSGD and LPSGDM, which are detailed in Subsections~\ref{method-2} and \ref{method-3}, respectively.

\subsection{Motivation}
\label{method-1}
Under the \(\ell_p\)-smoothness model, the steepest descent approach can be characterized as the minimizer of a local upper bound of the objective:
\begin{equation}
\theta_{t+1}
=
\arg\min_{\theta \in \mathbb{R}^d}
\left(
\langle \nabla f_t,\theta-\theta_t\rangle
+
\frac{L}{2}\|\theta-\theta_t\|_p^2
\right).
\label{eq:sd_lp_1}
\end{equation}
Since multiplying \eqref{eq:sd_lp_1} by the positive constant \(1/L\) does not change its minimizer, the above problem is equivalently written as
\begin{equation}
\theta_{t+1}
=
\arg\min_{\theta \in \mathbb{R}^d}
\left(
\left\langle \frac{1}{L}\nabla f_t,\theta-\theta_t \right\rangle
+
\frac{1}{2}\|\theta-\theta_t\|_p^2
\right).
\label{eq:sd_lp_2}
\end{equation}
In practice, replacing the full gradient by a stochastic gradient \(g_t\) and absorbing \(1/L\) into the step size \(\eta\), we consider the following stochastic \(\ell_p\)-norm update:
\begin{equation}
\theta_{t+1}
=
\arg\min_{\theta \in \mathbb{R}^d}
\left(
\langle \eta g_t,\theta-\theta_t\rangle
+
\frac{1}{2}\|\theta-\theta_t\|_p^2
\right).
\label{eq:sd_lp_3}
\end{equation}

The minimizer of \eqref{eq:sd_lp_3} admits a closed-form expression. Specifically, for each coordinate $i$, we have
\begin{equation}
\theta_{t+1}^{(i)}
=
\theta_t^{(i)}
-\eta \|g_t\|_q^{\frac{p-2}{p-1}}
\frac{g_t^{(i)}}{|g_t^{(i)}|^{\frac{p-2}{p-1}}},
\label{core}
\end{equation}

Further, inspired by the closed-form $\ell_p$ steepest descent update in \cite{stacey}, we define the following coordinate-wise scaling operator:
\begin{equation}
\label{lpsd}
\theta_{t+1}^{(i)}
=
\theta_t^{(i)}
-\eta 
\frac{g_t^{(i)}}{(|g_t^{(i)}|+\varepsilon)^{\frac{p-2}{p-1}}},
\end{equation}
where the factor \(\|g_t\|_q^{\frac{p-2}{p-1}}\) in \eqref{core} is a positive scalar common to all coordinates. As such, it only controls the overall update magnitude, but does not alter the coordinate-wise relative scaling induced by the \(\ell_p\) geometry. Accordingly, we retain the coordinate-wise scaling induced by the \(\ell_p\) geometry and incorporate the global factor into the learning rate. Here, \(\varepsilon>0\) is a small stability constant introduced to avoid singular behavior near zero and to improve analytical tractability in the convergence analysis.

 To comprehend the numerical behavior of different norm geometries under curvature anisotropy, we conduct experiments on the following local quadratic approximation for \(\theta \in \mathbb{R}^d\):

\begin{equation}\label{eq:quad}
f(\theta) = \frac{1}{2}\,\theta^\top H \theta.
\end{equation}
Without loss of generality, we work in the eigenbasis of the local Hessian so that the Hessian matrix becomes a diagonal matrix, i.e., $H=\mathrm{diag}(\lambda^{(1)},\dots,\lambda^{(d)})$. We further consider the anisotropic regime $\lambda^{(1)} \gg \lambda^{(2)} \ge \cdots \ge 0$, where the leading curvature direction is significantly sharper than the remaining directions.

Notice that the gradient of the model \eqref{eq:quad} is $\nabla f(\theta)=H\theta$, i.e.
\begin{equation}\label{eq:grad}
(\nabla f(\theta))^{(i)} = \lambda^{(i)} \theta^{(i)}.
\end{equation}

\begin{figure}[!t]
\centering
\subfloat[]{\includegraphics[width=3.5in]{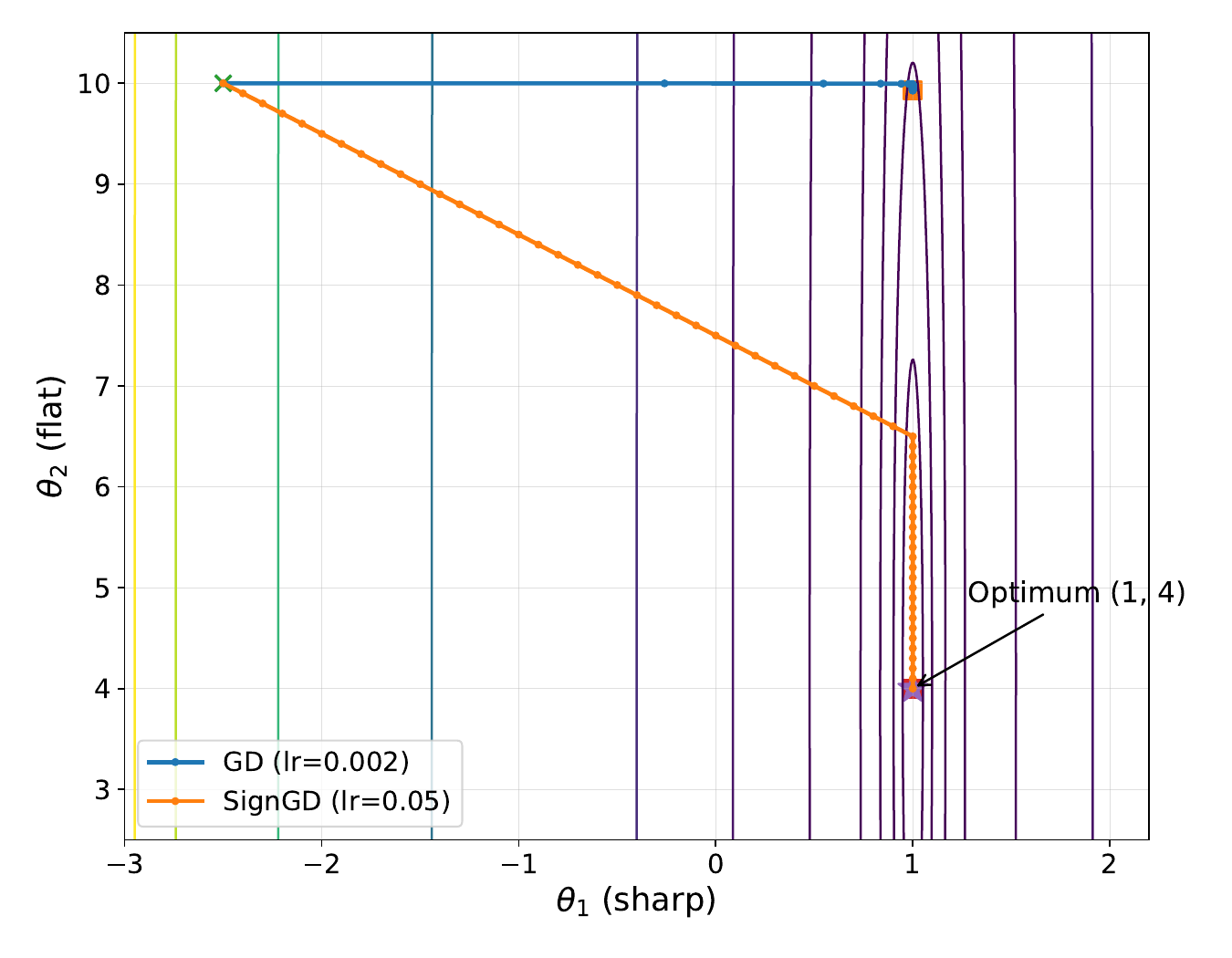}%
\label{fig_3_a}}
\hfil
\subfloat[]{\includegraphics[width=3.5in]{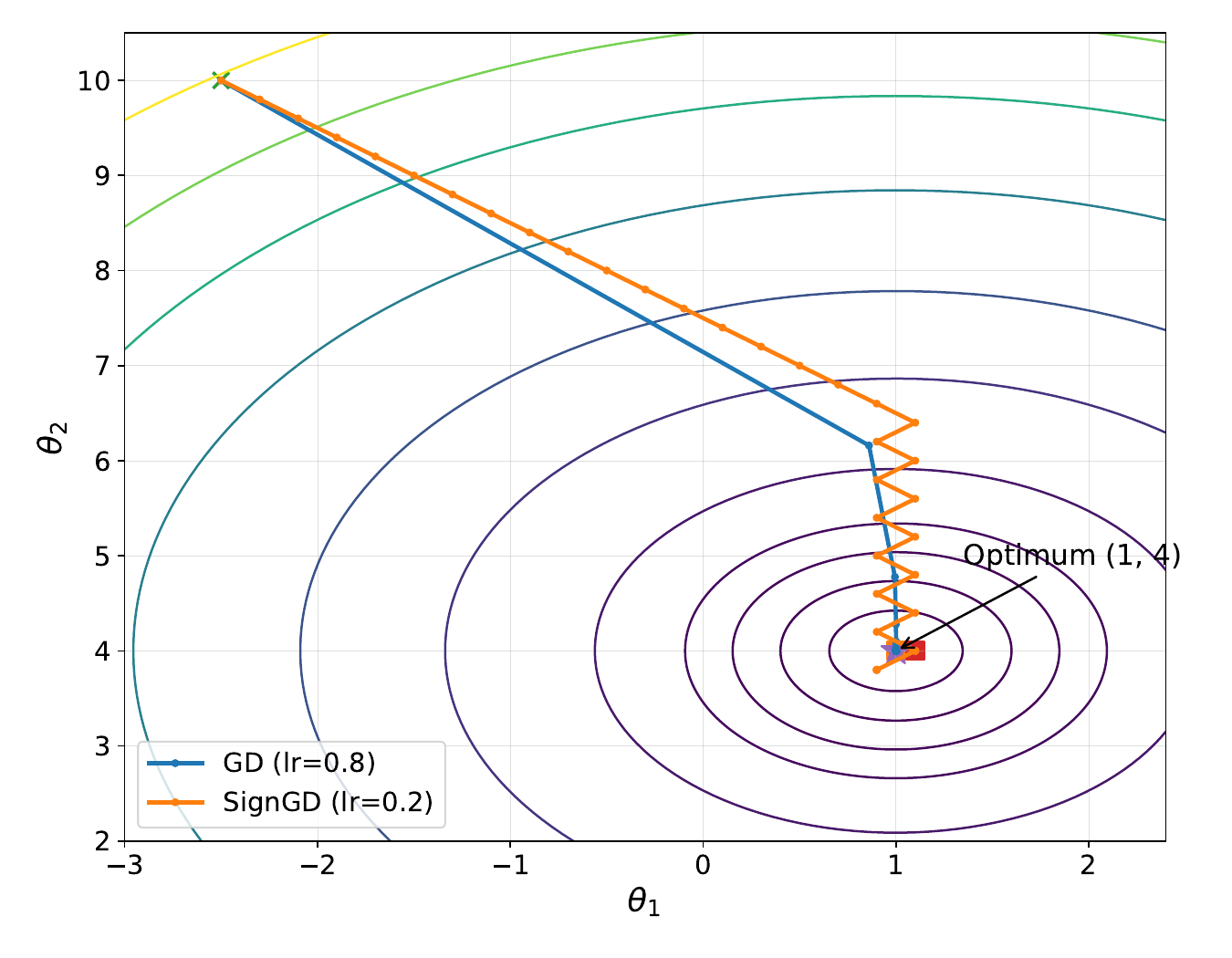}%
\label{fig_3_b}}
\caption{Optimization trajectories of GD and SignGD on two toy loss landscapes. (a) Trajectories of GD and SignGD under strong curvature anisotropy. (b) Trajectories of GD and SignGD under weak curvature anisotropy.}
\label{fig_3}
\end{figure}

When \(p=2\) (Euclidean geometry), the update magnitude along each coordinate satisfies
\begin{equation}
|d^{(i)}| \propto |(\nabla f(\theta))^{(i)}|=\lambda^{(i)}|\theta^{(i)}|.
\end{equation}
Hence,
\begin{equation}
\label{yz-1}
\frac{|d^{(i)}|}{|d^{(j)}|}
=
\frac{\lambda^{(i)}|\theta^{(i)}|}{\lambda^{(j)}|\theta^{(j)}|}.
\end{equation}
Therefore, as seen from \eqref{yz-1}, Euclidean geometry preserves directional imbalance linearly, and high-curvature directions can dominate the update.

Further, for \(2<p<\infty\), it follows from \eqref{lpsd} that the update magnitude admits a power-type rescaling, i.e.,
\begin{equation}
    |d^{(i)}| \propto |(\nabla f(\theta))^{(i)}|^\alpha \propto \bigl(\lambda^{(i)} |\theta^{(i)}|\bigr)^\alpha,
\end{equation}
where $\alpha=\frac{1}{p-1}$. Note that
when $\alpha<1$, the directional imbalance is compressed from linear to sublinear scaling. Consequently, the dominance of outlier curvature directions is weakened.

As $p\to\infty$, we have $\alpha\to 0$, and the update scheme of the optimizer approaches a sign-based form that depends only on the coordinate-wise signs of the gradient. This maximally suppresses imbalance of the update magnitude
along each coordinate, but also discards useful scale information.

To provide intuition, we consider two quadratic toy objectives. 
For the strongly anisotropic case, we use
\begin{equation}
   f_{\mathrm{strong}}(\theta^{(1)},\theta^{(2)})=\frac{200}{2}(\theta^{(1)}-1)^2+\frac{0.05}{2}(\theta^{(2)}-4)^2, 
\end{equation}
whose Hessian is $\mathrm{diag}(200,0.05)$. 
For the weakly anisotropic case, we use
\begin{equation}
    f_{\mathrm{weak}}(\theta^{(1)},\theta^{(2)})=\frac{1.2}{2}(\theta^{(1)}-1)^2+\frac{0.8}{2}(\theta^{(2)}-4)^2,
\end{equation}
whose Hessian is $\mathrm{diag}(1.2,0.8)$ and is close to isotropic.
We compare GD and signGD on both objectives and the corresponding optimization trajectories are presented in Fig.~\ref{fig_3}. In Fig.~\ref{fig_3_a}, which corresponds to the strongly anisotropic case, GD is restricted by the sharp direction and thus makes slow progress along the flat direction, whereas signGD traverses the flat direction more effectively because its updates are not scaled by gradient magnitude. In Fig.~\ref{fig_3_b}, which corresponds to the weakly anisotropic case, GD is better matched to the local geometry and converges smoothly to the optimum, while signGD oscillates near the minimizer due to its fixed coordinate-wise step size.

The above results from Fig. \ref{fig_3} indicate that different training stages may favor different norm geometries: a larger \(\ell_p\)-norm is preferable in the early stage when curvature anisotropy is strong, whereas a norm closer to \(\ell_2\) is more suitable in the later stage when the local curvature becomes flatter. This observation provides a key motivation for introducing cosine \(\ell_p\)-norm scheme.

\subsection{LPSGD}
\label{method-2}
We introduce a cosine \(\ell_p\)-norm scheme that gradually anneals the norm geometry from a larger norm toward the Euclidean norm. Specifically, at epoch \(s\), we set
\begin{equation}
p_s=2 + (p_{\max}-2)\cdot
\frac{1+\cos\!\left(\pi \frac{s-1}{S-1}\right)}{2},
\label{eq:p_schedule}
\end{equation}
where $p_{\max}\ge 2$ is the maximum norm parameter and $S$ is the total number of training epochs. This schedule starts from $p_{\max}$ at the beginning of training and smoothly decays to $2$ at the final epoch. Following, we will equip SGD and SGDM with the cosine $\ell_p$-norm scheme, respectively.

It is worth pointing out that to match the closed-form $\ell_p$ steepest-descent scaling, we define the epoch-dependent exponent
\begin{equation}
\rho_s
=
\frac{p_s-2}{p_s-1}.
\label{eq:rho_schedule}
\end{equation}
Given a stochastic gradient \(g_t \in \mathbb{R}^d\), we define the update direction through the nonlinear scaling map
\begin{equation}
\label{eq:ut}
u_t = \frac{g_t}{\left(|g_t|+\varepsilon\right)^{\rho_s}},
\end{equation}
The parameters in SGD are then updated as
\begin{equation}
\theta_{t+1} = \theta_t - \eta u_t,
\end{equation}
where \(\eta\) denotes the learning rate.

As discussed above, we are now ready to describe our first optimizer, LPSGD, in Algorithm \ref{LPSGD}.

\begin{algorithm}[t]
\caption{LPSGD}
\label{LPSGD}
\begin{algorithmic}[1]
\REQUIRE Initial parameter $\theta_0$, learning rate $\eta$, stability constant $\varepsilon>0$, maximum norm parameter $p_{\max}\ge 2$, total epochs $S$, training set $\mathcal{D}$
\ENSURE Final parameter $\theta$
\FOR{$s=1$ to $S$}
    \STATE $p_s \gets 2 + (p_{\max}-2)\cdot \dfrac{1+\cos\!\left(\pi \dfrac{s-1}{S-1}\right)}{2}$
    \STATE $\rho_s \gets \dfrac{p_s-2}{p_s-1}$
    \FOR{each mini-batch $\mathcal{B}_t \subset \mathcal{D}$}
        \STATE $g_t \leftarrow \nabla f_{\mathcal{B}_t}(\theta_t)$
        \STATE $u_t \gets \dfrac{g_t}{\left(|g_t|+\varepsilon\right)^{\rho_s}}$
        \STATE $\theta_{t+1} \gets \theta_t - \eta u_t$
    \ENDFOR
\ENDFOR
\end{algorithmic}
\end{algorithm}

\begin{remark} A few necessary explanations for LPSGD (Algorithm~\ref{LPSGD}) are described here.

\begin{itemize}

\item As mentioned above, the study \cite{stacey} also introduced $\ell_p$-norm into a new accelerated $\ell_p$ steepest descent algorithm, called STACEY. However, the authors used a fixed norm scheme, i.e., $p$ is a fixed constant. Such operation may be inadvisable for large models. In contrast,  LPSGD (Algorithm~\ref{LPSGD}) works with a dynamical norm scheme by introducing a cosine \(\ell_p\)-norm scheme.

\item Further, in recent work \cite{yuan-1}, the authors develop the gradual stochastic gradient descent (GSGD) algorithm, enabling the optimizer to smoothly transition from sign-based SGD in the early phase to standard SGD at the end by varying the optimizer from signSGD (equivalent to using the \(\ell_\infty\)-norm) to SGD (equivalent to using the \(\ell_2\)-norm), with \(p=\frac{t-1}{T}\), where \(T\) denotes the total number of iterations. However, GSGD is relatively less flexible, as the initial norm is implicitly fixed by its scheduling rule. By contrast, our LPSGD introduces a tunable parameter \(p_{\max}\), allowing the optimizer to start from a more suitable initial norm geometry for different models, thereby providing greater adaptability.

\item Looking back to SGD, we have that the computational cost of LPSGD (Algorithm \ref{LPSGD}) is almost similar to that of vanilla SGD. The only additional computation is the element-wise rescaling, whose cost is \(\mathcal{O}(d)\) for a \(d\)-dimensional parameter vector.

\end{itemize}

\end{remark}

\subsection{LPSGDM}
\label{method-3}
In practical deep learning training, momentum and weight decay are often crucial for improving optimization stability and generalization \cite{momentum1, momentum2}. We therefore extend Algorithm~\ref{LPSGD} to a more practical optimizer by incorporating both momentum and weight decay, thereby obtaining Algorithm \ref{LPSGDM} (LPSGDM).

At epoch $s$, the norm parameter $p_s$ is determined by the scheme in \eqref{eq:p_schedule}.
Given the stochastic gradient $g_t$, we first update the momentum buffer as
\begin{equation}
m_t = \beta m_{t-1} + (1-\beta) g_t,
\label{eq:momentum_buffer}
\end{equation}
where $\beta \in [0,1)$ denotes the momentum coefficient. We then apply the scheduled coordinate-wise rescaling to the momentum-smoothed gradient:
\begin{equation}
v_t = \frac{m_t}{\left(|m_t|+\varepsilon\right)^{\rho_s}},
\label{eq:momentum_scaled_update}
\end{equation}
where \(\varepsilon>0\) is the same stability constant as in the previous subsection.

To incorporate weight decay, we adopt the decoupled form, which applies shrinkage directly to the parameters rather than coupling it into the stochastic gradient. The resulting update is
\begin{equation}
\theta_{t+1} = (1-\eta\lambda)\theta_t - \eta v_t,
\label{eq:practical_update}
\end{equation}
where $\lambda \ge 0$ is the weight decay coefficient.

As discussed above, we are now ready to describe our second optimizer, LPSGDM, in Algorithm \ref{LPSGDM}.

\begin{algorithm}[!hpbt]
\caption{LPSGDM}
\label{LPSGDM}
\begin{algorithmic}[1]
\REQUIRE Initial parameter $\theta_0$, learning rate $\eta$, momentum coefficient $\beta\in[0,1)$, weight decay coefficient $\lambda\ge 0$, stability constant $\varepsilon>0$, maximum norm parameter $p_{\max}\ge 2$, total epochs $S$, training set $\mathcal{D}$
\ENSURE Final parameter $\theta$
\STATE Initialize momentum buffer $m_0 \gets 0$
\FOR{$s=1$ to $S$}
    \STATE $p_s \gets 2 + (p_{\max}-2)\cdot \dfrac{1+\cos\!\left(\pi \dfrac{s-1}{S-1}\right)}{2}$
    \STATE $\rho_s \gets \dfrac{p_s-2}{p_s-1}$
    \FOR{each mini-batch $\mathcal{B}_t \subset \mathcal{D}$}
        \STATE $g_t \leftarrow \nabla f_{\mathcal{B}_t}(\theta_t)$
        \STATE $m_t \gets \beta m_{t-1} + (1-\beta) g_t$
        \STATE $u_t \gets \dfrac{m_t}{\left(|m_t|+\varepsilon\right)^{\rho_s}}$
        \STATE $\theta_{t+1} \gets (1-\eta\lambda)\theta_t - \eta u_t$
    \ENDFOR
\ENDFOR
\end{algorithmic}
\end{algorithm}

\begin{remark}
Comparing LPSGD (Algorithm \ref{LPSGD}) and LPSGDM (Algorithm \ref{LPSGDM}), we have that the main difference between these two optimizers is that the latter introduces the momentum scheme. Specifically, LPSGDM accumulates historical gradient information through the momentum term, which can smooth the optimization trajectory.

\end{remark}

\section{Convergence Analysis}

In this section, we investigate the convergence behavior of the proposed methods in the nonconvex stochastic setting. We first establish the convergence guarantee for LPSGD (Algorithm~\ref{LPSGD}) in Subsection~\ref{convergence_analysis_1}. We then analyze the momentum-based extension without weight decay in Subsection~\ref{convergence_analysis_2}. Finally, the effect of the weight decay term is discussed separately in Subsection~\ref{convergence_analysis_3}.

\subsection{Convergence of LPSGD}
\label{convergence_analysis_1}

Let \(\{\theta_t\}_{t\ge 0}\) be the sequence generated by Algorithm~\ref{LPSGD} (LPSGD). 
Since the norm parameter is updated at the epoch level, let \(s(t)\) denote the epoch index corresponding to iteration \(t\). 
Accordingly, the active norm parameter, its dual exponent, and the associated rescaling exponent at iteration \(t\) are \(p_{s(t)}\), \(q_{s(t)}\), and \(\rho_{s(t)}\), respectively. 
The LPSGD update can be written as
\begin{equation}
u_t=\frac{g_t}{\left(|g_t|+\varepsilon\right)^{\rho_{s(t)}}}, 
\qquad
\theta_{t+1}=\theta_t-\eta u_t,
\label{eq:lpsgd_update_analysis}
\end{equation}
where \(\eta>0\) is the step size. 
Assumptions~\ref{ass:lpsmooth}--\ref{ass:lower} characterize the objective function and the stochastic gradient oracle. 
To handle the nonlinear coordinate-wise transformation induced by the scheduled \(\ell_p\)-geometry, we further impose the following condition on the transformed update direction.

\begin{assumption}[Conditional alignment and bounded transformed second moment]
\label{ass:direction}
There exist constants \(c_1,c_2,c_3>0\) such that, for any iteration \(t\),
\begin{equation}
\mathbb{E}\!\left[\left\langle \nabla f(\theta_t),u_t\right\rangle \mid \theta_t\right]
\ge c_1 \|\nabla f(\theta_t)\|_2^2,
\label{eq:alignment_assumption}
\end{equation}
and
\begin{equation}
\mathbb{E}\!\left[\|u_t\|_{p_{s(t)}}^2 \mid \theta_t\right]
\le c_2 + c_3 \|\nabla f(\theta_t)\|_2^2.
\label{eq:transformed_second_moment}
\end{equation}
\end{assumption}
Assumption~\ref{ass:direction} is needed because LPSGD updates the parameters using the transformed direction \(u_t\) rather than the stochastic gradient \(g_t\). Owing to the nonlinear coordinate-wise rescaling, the unbiasedness of \(g_t\) alone is not sufficient for the convergence analysis. Condition \eqref{eq:alignment_assumption} guarantees the descent capability of \(u_t\), condition \eqref{eq:transformed_second_moment} controls its size for bounding the smoothness term.
The following lemma establishes a one-step descent inequality for LPSGD.

\begin{lemma}
\label{lem:lpsgd_descent}
Under Assumption~\ref{ass:lpsmooth}, for any iteration \(t \ge 0\), the iterates generated by \eqref{eq:lpsgd_update_analysis} satisfy
\begin{equation}
f(\theta_{t+1})
\le
f(\theta_t)
-\eta \left\langle \nabla f(\theta_t),u_t\right\rangle
+\frac{L\eta^2}{2}\|u_t\|_{p_{s(t)}}^2.
\label{eq:one_step_descent}
\end{equation}
\end{lemma}

\begin{proof}
By Assumption~\ref{ass:lpsmooth}, applied with the active norm \(\|\cdot\|_{p_{s(t)}}\), we have
\[
f(\theta_{t+1})
\le
f(\theta_t)
+\left\langle \nabla f(\theta_t),\theta_{t+1}-\theta_t\right\rangle
+\frac{L}{2}\|\theta_{t+1}-\theta_t\|_{p_{s(t)}}^2.
\]
Substituting \(\theta_{t+1}-\theta_t=-\eta u_t\) yields
\[
f(\theta_{t+1})
\le
f(\theta_t)
-\eta \left\langle \nabla f(\theta_t),u_t\right\rangle
+\frac{L\eta^2}{2}\|u_t\|_{p_{s(t)}}^2.
\]
This completes the proof.
\end{proof}

We are now ready to state the main convergence result for LPSGD.

\begin{theorem}
\label{thm:lpsgd}
Suppose that Assumptions~\ref{ass:lpsmooth}--\ref{ass:lower} and Assumption~\ref{ass:direction} hold. 
Let \(\{\theta_t\}_{t=0}^{T-1}\) be generated by LPSGD with a constant step size \(\eta>0\). 
If
\begin{equation}
0<\eta\le \frac{c_1}{L c_3},
\label{eq:stepsize_condition}
\end{equation}
then
\begin{equation}
\frac{1}{T}\sum_{t=0}^{T-1}\mathbb{E}\|\nabla f(\theta_t)\|_2^2
\le
\frac{2\bigl(f(\theta_0)-f^\star\bigr)}{\eta c_1 T}
+\frac{L\eta c_2}{c_1}.
\label{eq:lpsgd_main_bound}
\end{equation}
Consequently, by choosing
\begin{equation}
\eta
=
\min\left\{
\frac{c_1}{L c_3},
\sqrt{\frac{2\bigl(f(\theta_0)-f^\star\bigr)}{L c_2 T}}
\right\},
\label{eq:lpsgd_stepsize_choice}
\end{equation}
we obtain
\begin{equation}
\frac{1}{T}\sum_{t=0}^{T-1}\mathbb{E}\|\nabla f(\theta_t)\|_2^2
=
\mathcal{O}(T^{-1/2}).
\label{eq:lpsgd_rate}
\end{equation}
\end{theorem}

\begin{proof}
Taking conditional expectation on both sides of \eqref{eq:one_step_descent} gives
\begin{align}
\mathbb{E}\!\left[f(\theta_{t+1})\mid \theta_t\right]
&\le
f(\theta_t)
-\eta\,\mathbb{E}\!\left[\left\langle \nabla f(\theta_t),u_t\right\rangle\mid \theta_t\right] \notag\\
&\quad
+\frac{L\eta^2}{2}\mathbb{E}\!\left[\|u_t\|_{p_{s(t)}}^2\mid \theta_t\right].
\label{eq:conditional_descent}
\end{align}
Applying Assumption~\ref{ass:direction}, we further obtain
\begin{align}
\mathbb{E}\!\left[f(\theta_{t+1})\mid \theta_t\right]
&\le
f(\theta_t)
-\eta c_1 \|\nabla f(\theta_t)\|_2^2 \notag\\
&\quad
+\frac{L\eta^2}{2}\Bigl(c_2+c_3\|\nabla f(\theta_t)\|_2^2\Bigr) \notag\\
&=
f(\theta_t)
-\left(\eta c_1-\frac{L\eta^2 c_3}{2}\right)\|\nabla f(\theta_t)\|_2^2 \notag\\
&\quad
+\frac{L\eta^2 c_2}{2}.
\label{eq:before_stepsize}
\end{align}
Under the step-size condition \eqref{eq:stepsize_condition}, we have
\[
\eta c_1-\frac{L\eta^2 c_3}{2}
\ge
\frac{\eta c_1}{2}.
\]
Hence,
\begin{equation}
\mathbb{E}\!\left[f(\theta_{t+1})\mid \theta_t\right]
\le
f(\theta_t)
-\frac{\eta c_1}{2}\|\nabla f(\theta_t)\|_2^2
+\frac{L\eta^2 c_2}{2}.
\label{eq:descent_after_stepsize}
\end{equation}
Taking full expectation on both sides of \eqref{eq:descent_after_stepsize} yields
\begin{equation}
\mathbb{E}\!\left[f(\theta_{t+1})\right]
\le
\mathbb{E}\!\left[f(\theta_t)\right]
-\frac{\eta c_1}{2}\,\mathbb{E}\!\left[\|\nabla f(\theta_t)\|_2^2\right]
+\frac{L\eta^2 c_2}{2}.
\label{eq:full_expectation_descent}
\end{equation}
Summing \eqref{eq:full_expectation_descent} from \(t=0\) to \(T-1\), we obtain
\begin{equation}
\mathbb{E}[f(\theta_T)]
\le
f(\theta_0)
-\frac{\eta c_1}{2}\sum_{t=0}^{T-1}\mathbb{E}[\|\nabla f(\theta_t)\|_2^2]
+\frac{L\eta^2 c_2 T}{2}.
\label{eq:telescope}
\end{equation}
By Assumption~\ref{ass:lower}, \(f(\theta_T)\ge f^\star\). Therefore,
\begin{equation}
    \frac{\eta c_1}{2}\sum_{t=0}^{T-1}\mathbb{E}[\|\nabla f(\theta_t)\|_2^2]
\le
f(\theta_0)-f^\star+\frac{L\eta^2 c_2 T}{2}.
\end{equation}
Dividing both sides by \(\eta c_1 T/2\) gives
\begin{equation}
    \frac{1}{T}\sum_{t=0}^{T-1}\mathbb{E}[\|\nabla f(\theta_t)\|_2^2]
\le
\frac{2\bigl(f(\theta_0)-f^\star\bigr)}{\eta c_1 T}
+\frac{L\eta c_2}{c_1},
\end{equation}
which proves \eqref{eq:lpsgd_main_bound}. 
Finally, substituting the step size in \eqref{eq:lpsgd_stepsize_choice} into \eqref{eq:lpsgd_main_bound} directly yields
\[
\frac{1}{T}\sum_{t=0}^{T-1}\mathbb{E}[\|\nabla f(\theta_t)\|_2^2]
=
\mathcal{O}(T^{-1/2}),
\]
thereby completing the proof.
\end{proof}

\subsection{Convergence of the Momentum-Based Extension}
\label{convergence_analysis_2}

We next consider the momentum-based extension of LPSGD without weight decay. 
Let \(\{\theta_t\}_{t\ge 0}\) be the sequence generated by Algorithm~\ref{LPSGDM}, and let \(\{m_t\}_{t\ge 0}\) denote the momentum sequence defined by
\begin{equation}
m_t=\beta m_{t-1}+(1-\beta)g_t,
\qquad 0\le \beta<1,
\label{eq:momentum_recursion}
\end{equation}
with \(m_{-1}=0\). 
Since the norm parameter is scheduled at the epoch level, let \(s(t)\) denote the epoch index corresponding to iteration \(t\). 
Accordingly, the active norm parameter and the associated rescaling exponent at iteration \(t\) are \(p_{s(t)}\) and \(\rho_{s(t)}\), respectively. 
Define the transformed momentum direction by
\begin{equation}
v_t=\frac{m_t}{\left(|m_t|+\varepsilon\right)^{\rho_{s(t)}}},
\qquad
\theta_{t+1}=\theta_t-\eta v_t,
\label{eq:lpsgdm_update_analysis}
\end{equation}
where \(\eta>0\) is the step size.

Since the update direction now depends on the momentum buffer, conditioning only on \(\theta_t\) is no longer sufficient. 
Therefore, let \(\{\mathcal{F}_t\}_{t\ge 0}\) denote the natural filtration generated by the algorithmic randomness up to iteration \(t\). 
To characterize the transformed momentum direction used by LPSGDM, we impose the following assumption.

\begin{assumption}[Conditional alignment and bounded transformed second moment for momentum direction]
\label{ass:momentum_direction}
There exist constants \(\tilde c_1,\tilde c_2,\tilde c_3>0\) such that, for any iteration \(t\),
\begin{equation}
\mathbb{E}\!\left[\left\langle \nabla f(\theta_t),v_t\right\rangle \mid \mathcal{F}_t\right]
\ge \tilde c_1 \|\nabla f(\theta_t)\|_2^2,
\label{eq:momentum_alignment}
\end{equation}
and
\begin{equation}
\mathbb{E}\!\left[\|v_t\|_{p_{s(t)}}^2 \mid \mathcal{F}_t\right]
\le \tilde c_2 + \tilde c_3 \|\nabla f(\theta_t)\|_2^2.
\label{eq:momentum_second_moment}
\end{equation}
\end{assumption}

Assumption~\ref{ass:momentum_direction} plays the same role as Assumption~\ref{ass:direction} in the analysis of LPSGD, but is adapted to the momentum-smoothed update direction. 
Because the actual update is generated from the transformed momentum \(v_t\) rather than the stochastic gradient \(g_t\), the unbiasedness of \(g_t\) alone does not directly imply a descent property for \(v_t\). 
Condition~\eqref{eq:momentum_alignment} guarantees that the transformed momentum direction remains sufficiently aligned with the full gradient in conditional expectation, condition \eqref{eq:momentum_second_moment} controls its magnitude for bounding the smoothness term.

The following lemma establishes a one-step descent inequality for the momentum-based method.

\begin{lemma}
\label{lem:lpsgdm_descent}
Under Assumption~\ref{ass:lpsmooth}, for any iteration \(t\ge 0\), the iterates generated by \eqref{eq:lpsgdm_update_analysis} satisfy
\begin{equation}
f(\theta_{t+1})
\le
f(\theta_t)
-\eta \left\langle \nabla f(\theta_t),v_t\right\rangle
+\frac{L\eta^2}{2}\|v_t\|_{p_{s(t)}}^2.
\label{eq:momentum_one_step_descent}
\end{equation}
\end{lemma}

\begin{proof}
By Assumption~\ref{ass:lpsmooth}, applied with the active norm \(\|\cdot\|_{p_{s(t)}}\), we have
\[
f(\theta_{t+1})
\le
f(\theta_t)
+\left\langle \nabla f(\theta_t),\theta_{t+1}-\theta_t\right\rangle
+\frac{L}{2}\|\theta_{t+1}-\theta_t\|_{p_{s(t)}}^2.
\]
Substituting \(\theta_{t+1}-\theta_t=-\eta v_t\) yields
\[
f(\theta_{t+1})
\le
f(\theta_t)
-\eta \left\langle \nabla f(\theta_t),v_t\right\rangle
+\frac{L\eta^2}{2}\|v_t\|_{p_{s(t)}}^2.
\]
This completes the proof.
\end{proof}

We are now ready to state the convergence result for the momentum-based extension.

\begin{theorem}
\label{thm:lpsgdm}
Suppose that Assumptions~\ref{ass:lpsmooth}--\ref{ass:lower} and Assumption~\ref{ass:momentum_direction} hold. 
Then there exists a constant \(\eta_0>0\) such that, for any fixed step size \(0<\eta\le \eta_0\), the iterates \(\{\theta_t\}_{t=0}^{T-1}\) generated by \eqref{eq:lpsgdm_update_analysis} satisfy
\begin{equation}
\frac{1}{T}\sum_{t=0}^{T-1}\mathbb{E}\!\left[\|\nabla f(\theta_t)\|_2^2\right]
\le
\frac{2\bigl(f(\theta_0)-f^\star\bigr)}{\eta \tilde c_1 T}
+\frac{L\eta \tilde c_2}{\tilde c_1}.
\label{eq:lpsgdm_main_bound}
\end{equation}
Consequently, by choosing
\begin{equation}
\eta=
\min\left\{
\eta_0,\,
\sqrt{\frac{2\bigl(f(\theta_0)-f^\star\bigr)}{L\tilde c_2 T}}
\right\},
\label{eq:lpsgdm_stepsize_choice}
\end{equation}
we obtain
\begin{equation}
\frac{1}{T}\sum_{t=0}^{T-1}\mathbb{E}\!\left[\|\nabla f(\theta_t)\|_2^2\right]
=
\mathcal{O}(T^{-1/2}).
\label{eq:lpsgdm_rate}
\end{equation}
One admissible choice is \(\eta_0=\tilde c_1/(L\tilde c_3)\).
\end{theorem}

\begin{proof}
Taking conditional expectation on both sides of \eqref{eq:momentum_one_step_descent} gives
\begin{align}
\mathbb{E}\!\left[f(\theta_{t+1})\mid \mathcal{F}_t\right]
&\le
f(\theta_t)
-\eta\,\mathbb{E}\!\left[\left\langle \nabla f(\theta_t),v_t\right\rangle \mid \mathcal{F}_t\right] \notag\\
&\quad
+\frac{L\eta^2}{2}\,\mathbb{E}\!\left[\|v_t\|_{p_{s(t)}}^2 \mid \mathcal{F}_t\right].
\label{eq:momentum_conditional_descent}
\end{align}
Applying Assumption~\ref{ass:momentum_direction}, we obtain
\begin{align}
\mathbb{E}\!\left[f(\theta_{t+1})\mid \mathcal{F}_t\right]
&\le
f(\theta_t)
-\eta \tilde c_1 \|\nabla f(\theta_t)\|_2^2 \notag\\
&\quad
+\frac{L\eta^2}{2}\Bigl(\tilde c_2+\tilde c_3\|\nabla f(\theta_t)\|_2^2\Bigr) \notag\\
&=
f(\theta_t)
-\left(\eta \tilde c_1-\frac{L\eta^2\tilde c_3}{2}\right)\|\nabla f(\theta_t)\|_2^2 \notag\\
&\quad
+\frac{L\eta^2\tilde c_2}{2}.
\label{eq:momentum_before_stepsize}
\end{align}
For any \(0<\eta\le \eta_0=\tilde c_1/(L\tilde c_3)\), we have
\[
\eta \tilde c_1-\frac{L\eta^2\tilde c_3}{2}
\ge
\frac{\eta \tilde c_1}{2}.
\]
Hence,
\begin{equation}
\mathbb{E}\!\left[f(\theta_{t+1})\mid \mathcal{F}_t\right]
\le
f(\theta_t)
-\frac{\eta \tilde c_1}{2}\|\nabla f(\theta_t)\|_2^2
+\frac{L\eta^2\tilde c_2}{2}.
\label{eq:momentum_after_stepsize}
\end{equation}
Taking full expectation on both sides of \eqref{eq:momentum_after_stepsize} yields
\begin{equation}
\mathbb{E}\!\left[f(\theta_{t+1})\right]
\le
\mathbb{E}\!\left[f(\theta_t)\right]
-\frac{\eta \tilde c_1}{2}\mathbb{E}\!\left[\|\nabla f(\theta_t)\|_2^2\right]
+\frac{L\eta^2\tilde c_2}{2}.
\label{eq:momentum_full_expectation}
\end{equation}
Summing \eqref{eq:momentum_full_expectation} from \(t=0\) to \(T-1\), we obtain
\begin{equation}
\mathbb{E}\!\left[f(\theta_T)\right]
\le
f(\theta_0)
-\frac{\eta \tilde c_1}{2}\sum_{t=0}^{T-1}\mathbb{E}\!\left[\|\nabla f(\theta_t)\|_2^2\right]
+\frac{L\eta^2\tilde c_2 T}{2}.
\label{eq:momentum_telescope}
\end{equation}
By Assumption~\ref{ass:lower}, \(f(\theta_T)\ge f^\star\). Therefore,
\[
\frac{\eta \tilde c_1}{2}\sum_{t=0}^{T-1}\mathbb{E}\!\left[\|\nabla f(\theta_t)\|_2^2\right]
\le
f(\theta_0)-f^\star+\frac{L\eta^2\tilde c_2 T}{2}.
\]
Dividing both sides by \(\eta \tilde c_1 T/2\) gives
\[
\frac{1}{T}\sum_{t=0}^{T-1}\mathbb{E}\!\left[\|\nabla f(\theta_t)\|_2^2\right]
\le
\frac{2\bigl(f(\theta_0)-f^\star\bigr)}{\eta \tilde c_1 T}
+\frac{L\eta \tilde c_2}{\tilde c_1},
\]
which proves \eqref{eq:lpsgdm_main_bound}. 
Finally, substituting \eqref{eq:lpsgdm_stepsize_choice} into \eqref{eq:lpsgdm_main_bound} yields \eqref{eq:lpsgdm_rate}. 
This completes the proof.
\end{proof}

\subsection{Discussion on Decoupled Weight Decay}

\label{convergence_analysis_3}

In practice, we also consider a decoupled weight decay version of the momentum-based method, whose update is given by
\begin{equation}
\theta_{t+1} = (1-\eta\lambda)\theta_t - \eta v_t,
\label{eq:lpsgdmw_update}
\end{equation}
where \(\lambda>0\) denotes the weight decay coefficient.

The decoupled weight decay term introduces an additional deterministic contraction on the parameter vector, while leaving the transformed momentum direction \(v_t\) unchanged. 
Hence, it does not modify the core geometry of the proposed dynamic \(\ell_p\)-norm update, but serves as an auxiliary regularization mechanism in practical training.

A fully rigorous treatment of decoupled weight decay would require additional assumptions to control the parameter norm along the whole trajectory. 
Since our main goal is to characterize the optimization effect of the dynamic \(\ell_p\)-norm scheduler, we focus the formal analysis on the cases without weight decay. 
Nevertheless, from an algorithmic viewpoint, decoupled weight decay is compatible with the proposed framework and can be incorporated straightforwardly in implementation.

\section{Experiments}
In this section, we assess the effectiveness of our algorithms on image classification benchmarks using networks of different capacities and depths, including VGG-11\cite{vgg11}, ResNet-18, and ResNet-50\cite{resnet}. On the small-scale benchmarks CIFAR-10 and CIFAR-100\cite{cifar}, we report results for all three networks. On the large-scale benchmark ImageNet-1K\cite{imagenet}, we further evaluate results by conducting experiments on ResNet-50. We compare our algorithms against several  state-of-the-art optimizers, including SGD w/Momentum\cite{sgdm}, AdamW\cite{adamw}, and Lion \cite{lion}.

\textbf{Models}. Several explanations for the used deep learning models will be briefly discussed here.
VGG-11 is a plain deep convolutional network with a simple stacked architecture, serving as a representative non-residual baseline. ResNet-18 is a lightweight residual network, which improves optimization and feature learning through skip connections. ResNet-50 is a deeper and higher-capacity residual model, allowing us to examine the scalability of the proposed method on more challenging architectures.

\textbf{Datasets}. CIFAR-10 and CIFAR-100 consist of $32\times 32$ natural color images with 10 and 100 classes, respectively, serving as small-resolution benchmarks. ImageNet-1K is a large-scale dataset of high-resolution natural images spanning 1000 classes, providing a more challenging setting to test scalability and generalization. To accommodate the different input resolutions and label spaces across datasets, we adapt each model to the corresponding dataset by adjusting input and output layers accordingly. In particular, we adjust the early convolutional layers according to the input resolution of each dataset, and modify the final classification layer to match the number of classes ($C \in \{10, 100, 1000\}$).

\textbf{Hyperparameter Tuning.} For SGD w/Momentum and LPSGDM, we perform grid search over the momentum coefficient $\beta_1 \in \{0, 0.9\}$, the learning rate $\eta \in \{0, 0.1, 0.05, 0.01, 0.002, 0.001, 0.0005\}$, and the weight decay $\lambda \in \{0,0.1, 0.01,0.005,0.001,0.0005,0.0001,0.00001\}$. Note that when $\beta_1=0$ and $\lambda=0$, SGD w/Momentum and LPSGDM reduce to SGD and LPSGD, respectively. Therefore, we do not report SGD and LPSGD separately in the final experimental results. 

For LPSGDM, we additionally tune the parameter \(p_{\max}\). Fig.~4 shows the testing top-1 accuracy versus training epoch under different \(p_{\max}\) settings on ResNet-18/CIFAR-10 and ResNet-50/ImageNet-1K, respectively. The results indicate that the testing top-1 accuracy does not increase monotonically with \(p_{\max}\). Instead, excessively large values of $p_{\max}$ may lead to fluctuations or even degraded optimization performance. Moreover, different architectures exhibit different degrees of curvature anisotropy, and the corresponding optimal choice of $p_{\max}$ also varies across models.

\begin{figure*}[t]
\centering
\subfloat[]{\includegraphics[width=3.5in]{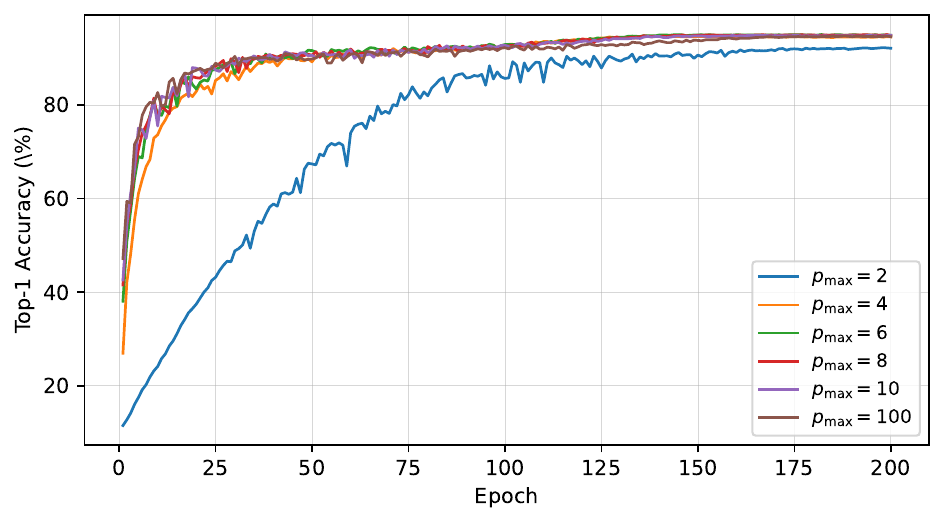}%
\label{fig_first_case}}
\hfil
\subfloat[]{\includegraphics[width=3.5in]{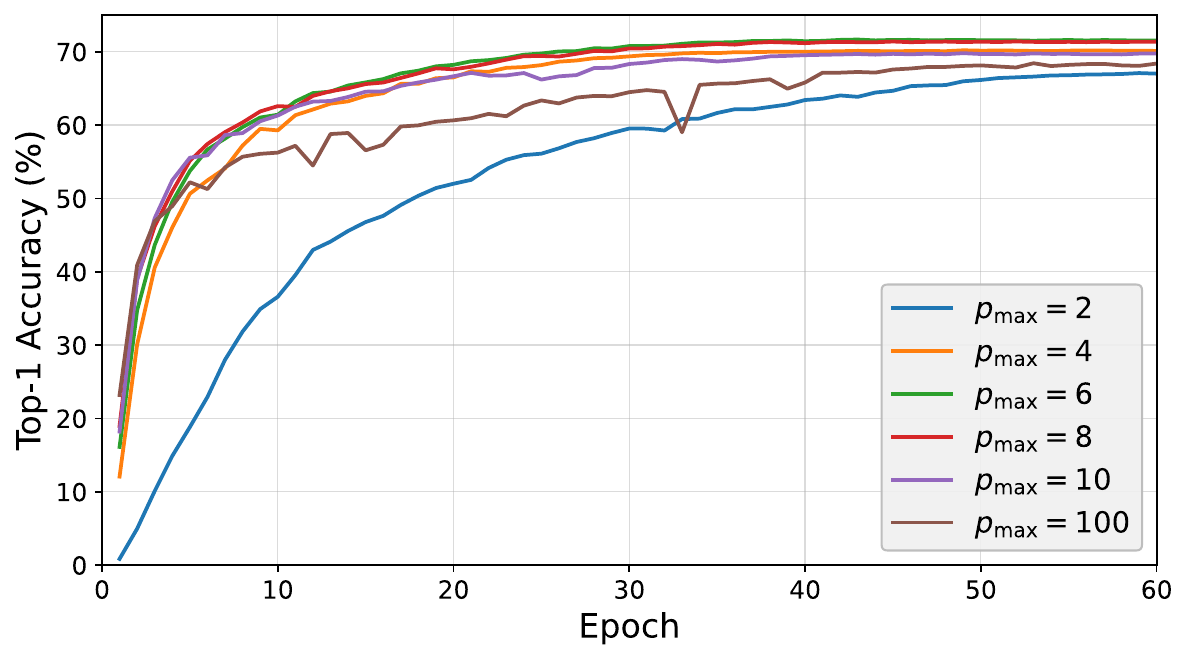}%
\label{fig_second_case}}
\caption{(a) Testing top-1 accuracy curves on ResNet-18 / cifar-10 under different \(p_{max}\) settings. (b) Testing top-1 accuracy curves on ResNet-50 / ImageNet-1K under different \(p_{max}\) settings.}
\label{fig_sim}
\end{figure*}

For AdamW, we fix $\beta_1=0.9$ and $\beta_2=0.999$, and perform grid search over the learning rate $\eta \in \{0.01, 0.005, 0.001, 0.0005, 0.0003, 0.0001\}$ and the weight decay $\lambda \in \{0.1,0.01,0.001,0.0005,0.0001,0.00001\}$. For Lion, we fix $\beta_1=0.9$ and $\beta_2=0.99$, and tune the learning rate and weight decay over the same search space as adamW.

The complete hyperparameter settings used in all experiments are provided in \ref{HYPER-PARAMETERS}.

\textbf{Result.} Tables~\ref{tab:cifar10_epochs}--\ref{tab:imagenet1k_epochs} summarize the performance of different optimizers in terms of training negative log-likelihood (NLL \(\downarrow\)) and testing top-1 accuracy (\(\uparrow\)) at selected training checkpoints. Specifically, Tables~\ref{tab:cifar10_epochs} and \ref{tab:cifar100_epochs} report the results on CIFAR-10 and CIFAR-100 at the 50th, 100th, and 200th epochs, respectively, while Table~\ref{tab:imagenet1k_epochs} reports the results on ImageNet-1K at the 20th, 40th, and 60th epochs.

Tables~\ref{tab:cifar10_epochs} and \ref{tab:cifar100_epochs} report the results on CIFAR-10 and CIFAR-100, respectively. On CIFAR-10, LPSGDM consistently delivers the strongest generalization performance across all three architectures and all evaluation checkpoints. In particular, at 200 epochs, it achieves 90.73\%, 94.81\%, and 94.84\% top-1 accuracy on VGG-11, ResNet-18, and ResNet-50, respectively, outperforming SGD w/Momentum, AdamW, and Lion in all cases. Moreover, the advantage is already visible at 50 and 100 epochs, indicating that the proposed cosine \(\ell_p\)-norm scheme improves not only the final accuracy but also the generalization behavior throughout training. On CIFAR-100, where the task is more challenging, the benefit of LPSGDM becomes even more evident in the middle and late stages. Although Lion achieves competitive accuracy at several early checkpoints, LPSGDM attains the best final top-1 accuracy on all three backbones, reaching 60.46\% on VGG-11, 77.98\% on ResNet-18, and 76.78\% on ResNet-50 at 200 epochs. In particular, on ResNet-50, LPSGDM already becomes the best method at 100 epochs with 71.45\% and further enlarges the margin at 200 epochs. These results indicate that LPSGDM yields a more favorable generalization.

Table~\ref{tab:imagenet1k_epochs} further demonstrates that LPSGDM maintains strong scalability and superior generalization on the large-scale ImageNet-1K benchmark. Although AdamW obtains the highest top-1 accuracy at the 20th epoch, LPSGDM overtakes all baselines in the middle and late stages, achieving 73.48\% at epoch 40 and 74.04\% at epoch 60 on ResNet-50. These results surpass SGD w/Momentum, AdamW, and Lion by clear margins. In particular, the improvement over AdamW reaches 2.83 percentage points at epoch 40 and remains 2.25 percentage points at epoch 60. It is also worth noting that LPSGDM does not always correspond to the smallest training NLL, yet it still produces the highest testing accuracy. This suggests that the proposed cosine \(\ell_p\)-norm scheme improves the generalization quality of the learned solution rather than simply driving down the training loss. Overall, the results in Tables~I--III consistently show that LPSGDM delivers strong generalization performance across different datasets and network architectures.

\begin{table}[t]
\centering
\caption{Image classification on CIFAR-10 at the 50th, 100th, and 200th epochs.}
\label{tab:cifar10_epochs}

\setlength{\tabcolsep}{4pt}
\renewcommand{\arraystretch}{1.1}
\makebox[\textwidth][c]{%
\resizebox{1.25\textwidth}{!}{%
\begin{tabular}{c|c|ccc|ccc}
\hline
\multicolumn{1}{c|}{\textbf{Model}} &
\multicolumn{1}{c|}{\textbf{Optimizer}} &
\multicolumn{3}{c|}{\textbf{Training NLL} $\downarrow$} &
\multicolumn{3}{c}{\textbf{Testing Top-1 ACC (\%)} $\uparrow$} \\
\cline{3-8}
& & \textbf{@50 epoch} & \textbf{@100 epoch} & \textbf{@200 epoch} & \textbf{@50 epoch} & \textbf{@100 epoch} & \textbf{@200 epoch} \\
\hline

VGG-11 & SGD w/ Momentum & 0.2610 $\pm$ 0.0101 & 0.1001 $\pm$ 0.0059 & 0.0096 $\pm$ 0.0001 & 85.88 $\pm$ 0.16 & 88.43 $\pm$ 0.40 & 90.05 $\pm$ 0.12 \\
VGG-11 & AdamW          & 0.4183 $\pm$ 0.2232 & 0.1574 $\pm$ 0.0698 & \textbf{0.0047 $\pm$ 0.0056} & 83.54 $\pm$ 3.87 & 87.83 $\pm$ 0.63 & 89.93 $\pm$ 0.18 \\
VGG-11 & Lion           & \textbf{0.2292 $\pm$ 0.0190} & 1.0760 $\pm$ 1.0448 & 0.0335 $\pm$ 0.0030 & 85.05 $\pm$ 1.22 & 83.79 $\pm$ 5.51 & 89.40 $\pm$ 0.11 \\
VGG-11 & LPSGDM         & 0.2679 $\pm$ 0.0535 & \textbf{0.0850 $\pm$ 0.0010} & 0.0142 $\pm$ 0.0019 & \textbf{86.15 $\pm$ 0.43} & \textbf{88.83 $\pm$ 0.41} & \textbf{90.73 $\pm$ 0.37} \\
\hline

ResNet-18 & SGD w/ Momentum & 0.1313 $\pm$ 0.0061 & 0.0175 $\pm$ 0.0017 & 0.0014 $\pm$ 0.0000 & 88.77 $\pm$ 0.41 & 91.70 $\pm$ 0.25 & 93.15 $\pm$ 0.09 \\
ResNet-18 & AdamW           & 0.0618 $\pm$ 0.0020 & 0.0107 $\pm$ 0.0000 & \textbf{0.0001 $\pm$ 0.0000} & 91.78 $\pm$ 0.29 & 93.09 $\pm$ 0.01 & 94.09 $\pm$ 0.17 \\
ResNet-18 & Lion            & 0.1101 $\pm$ 0.0006 & 0.0537 $\pm$ 0.0001 & 0.0002 $\pm$ 0.0001 & 90.53 $\pm$ 0.68 & 91.58 $\pm$ 0.05 & 94.04 $\pm$ 0.07 \\
ResNet-18 & LPSGDM          & \textbf{0.0498 $\pm$ 0.0027} & \textbf{0.0021 $\pm$ 0.0000} & 0.0004 $\pm$ 0.0002 & \textbf{92.12 $\pm$ 0.49} & \textbf{93.68 $\pm$ 0.02} & \textbf{94.81 $\pm$ 0.19} \\
\hline

ResNet-50 & SGD w/ Momentum & 0.13206 $\pm$ 0.00147 & 0.02201 $\pm$ 0.00207 & 0.00062 $\pm$ 0.00004 & 87.57 $\pm$ 0.82 & 90.73 $\pm$ 0.06 & 92.72 $\pm$ 0.16 \\
ResNet-50 & AdamW           & 0.07910 $\pm$ 0.00428 & 0.01305 $\pm$ 0.00124 & 0.00013 $\pm$ 0.00008 & 91.65 $\pm$ 0.51 & 93.45 $\pm$ 0.04 & 94.57 $\pm$ 0.02 \\
ResNet-50 & Lion            & 0.08016 $\pm$ 0.00207 & 0.01733 $\pm$ 0.00163 & \textbf{0.00010 $\pm$ 0.00012} & 91.15 $\pm$ 0.14 & 92.87 $\pm$ 0.11 & 94.34 $\pm$ 0.31 \\
ResNet-50 & LPSGDM          & \textbf{0.04780 $\pm$ 0.00035} & \textbf{0.00206 $\pm$ 0.00015} & 0.00070 $\pm$ 0.00067 & \textbf{92.22 $\pm$ 0.71} & \textbf{94.07 $\pm$ 0.08} & \textbf{94.84 $\pm$ 0.15} \\
\hline
\end{tabular}%
}
}
\end{table}

\begin{table*}[t]
\centering
\caption{Image classification on CIFAR-100 at the 50th, 100th, and 200th epochs.}
\label{tab:cifar100_epochs}

\setlength{\tabcolsep}{7.25pt}
\renewcommand{\arraystretch}{1.15}

\makebox[\textwidth][c]{%
\resizebox{1.25\textwidth}{!}{%
\begin{tabular}{c|c|ccc|ccc}
\hline
\multicolumn{1}{c|}{\textbf{Model}} &
\multicolumn{1}{c|}{\textbf{Optimizer}} & 
\multicolumn{3}{c|}{\textbf{Training NLL} $\downarrow$} & 
\multicolumn{3}{c}{\textbf{Testing Top-1 ACC (\%)} $\uparrow$} \\
\cline{3-8}
& & \textbf{@50 epoch} & \textbf{@100 epoch} & \textbf{@200 epoch} & \textbf{@50 epoch} & \textbf{@100 epoch} & \textbf{@200 epoch} \\
\hline

VGG-11 & SGD w/ Momentum & 1.6509 $\pm$ 0.0114 & \textbf{0.5577 $\pm$ 0.0169} & \textbf{0.0400 $\pm$ 0.0000} & 54.21 $\pm$ 1.07 & \textbf{58.33 $\pm$ 0.23} & 60.16 $\pm$ 0.29  \\
VGG-11 & AdamW           & 1.8164 $\pm$ 0.0234 & 1.3753 $\pm$ 0.0093 & 0.4451 $\pm$ 0.0071 & 51.01 $\pm$ 0.48 & 55.11 $\pm$ 1.03 & 59.34 $\pm$ 0.55 \\
VGG-11 & Lion            & \textbf{1.0459 $\pm$ 0.0148} & 0.6791 $\pm$ 0.0162 & 0.2964 $\pm$ 0.0171 & \textbf{55.92 $\pm$ 0.16} & 56.27 $\pm$ 0.36 & 58.74 $\pm$ 0.21 \\
VGG-11 & LPSGDM          & 1.2445 $\pm$ 0.0133 & 0.6810 $\pm$ 0.0263 & 0.2754 $\pm$ 0.0122 & 55.48 $\pm$ 0.10 & 57.78 $\pm$ 0.17 & \textbf{60.46 $\pm$ 0.25}\\
\hline

ResNet-18 & SGD w/ Momentum & 1.8969 $\pm$ 0.0476 & 1.5946 $\pm$ 0.0713 & 0.2105 $\pm$ 0.0129 & 44.12 $\pm$ 6.99 & 53.14 $\pm$ 1.83 & 76.86 $\pm$ 0.38 \\
ResNet-18 & AdamW           & 1.7273 $\pm$ 0.0028 & 1.4463 $\pm$ 0.0052 & 0.0326 $\pm$ 0.0004 & 47.92 $\pm$ 1.88 & 53.85 $\pm$ 3.72 & 74.23 $\pm$ 0.38 \\
ResNet-18 & Lion            & \textbf{0.4636 $\pm$ 0.0037} & \textbf{0.2078 $\pm$ 0.0009} & \textbf{0.0032 $\pm$ 0.0005} & \textbf{65.39 $\pm$ 1.00} & \textbf{67.64 $\pm$ 0.28} & 73.07 $\pm$ 0.11 \\
ResNet-18 & LPSGDM         & 1.7152 $\pm$ 0.0011 & 1.4065 $\pm$ 0.0035 & 1.2290 $\pm$ 0.0022 & 48.03 $\pm$ 1.92 & 54.31 $\pm$ 3.33 & \textbf{77.98 $\pm$ 0.35} \\
\hline

ResNet-50 & SGD w/ Momentum & 0.5314 $\pm$ 0.0077 & \textbf{0.0507 $\pm$ 0.0053} & 0.0083 $\pm$ 0.0003 & 63.34 $\pm$ 0.65 & 68.08 $\pm$ 0.25 & 69.70 $\pm$ 0.40 \\
ResNet-50 & AdamW           & 0.5210 $\pm$ 0.0056 & 0.2148 $\pm$ 0.0023 & \textbf{0.0021 $\pm$ 0.0003} & 65.84 $\pm$ 0.71 & 67.55 $\pm$ 0.70 & 74.91 $\pm$ 0.41 \\
ResNet-50 & Lion            & \textbf{0.2763 $\pm$ 0.0114} & 0.0744 $\pm$ 0.0027 & 0.0030 $\pm$ 0.0004 & \textbf{68.00 $\pm$ 1.10} & 70.39 $\pm$ 0.60 & 72.99 $\pm$ 0.51 \\
ResNet-50 & LPSGDM         & 0.4826 $\pm$ 0.0059 & 0.0888 $\pm$ 0.0018 & 0.0885 $\pm$ 0.0004 & 67.47 $\pm$ 0.52 & \textbf{71.45 $\pm$ 0.24} & \textbf{76.78 $\pm$ 0.17} \\
\hline

\end{tabular}
}
}
\end{table*}

\begin{table*}[!t]
\centering
\caption{Image classification on ImageNet-1K at the 20th, 40th, and 60th epochs.}
\label{tab:imagenet1k_epochs}

\setlength{\tabcolsep}{7pt}
\renewcommand{\arraystretch}{1.15}

\makebox[\textwidth][c]{%
\resizebox{1.25\textwidth}{!}{%
\begin{tabular}{c|c|ccc|ccc}
\hline
\multicolumn{1}{c|}{\textbf{Model}} &
\multicolumn{1}{c|}{\textbf{Optimizer}} & 
\multicolumn{3}{c|}{\textbf{Training NLL} $\downarrow$} & 
\multicolumn{3}{c}{\textbf{Testing Top-1 ACC (\%)} $\uparrow$} \\
\cline{3-8}
& & \textbf{@20 epoch} & \textbf{@40 epoch} & \textbf{@60 epoch} & \textbf{@20 epoch} & \textbf{@40 epoch} & \textbf{@60 epoch} \\
\hline
ResNet-50 & SGD w/ Momentum & 2.5096 $\pm$ 0.0086 & 1.8504 $\pm$ 0.0256 & 1.5996 $\pm$ 0.0239 & 52.19 $\pm$ 0.18 & 63.32 $\pm$ 0.11 & 67.16 $\pm$ 0.16 \\
ResNet-50 & AdamW          & \textbf{1.5629 $\pm$ 0.0023} & 1.2283 $\pm$ 0.0244 & \textbf{1.0755 $\pm$ 0.0155} & \textbf{67.80 $\pm$ 0.09} & 70.65 $\pm$ 0.55 & 71.79 $\pm$ 0.35 \\
ResNet-50 & Lion           & 1.7574 $\pm$ 0.0114 & 1.4263 $\pm$ 0.0781 & 1.1685 $\pm$ 0.0194 & 65.23 $\pm$ 0.66 & 67.76 $\pm$ 3.61 & 72.15 $\pm$ 0.21 \\
ResNet-50 & LPSGDM        & 1.6447 $\pm$ 0.0047 & \textbf{1.1491 $\pm$ 0.0031} & 1.1253 $\pm$ 0.0153 & 66.26 $\pm$ 0.54 & \textbf{73.48 $\pm$ 0.16} & \textbf{74.04 $\pm$ 0.10} \\
\hline
\end{tabular}
}
}
\end{table*}

\section{Conclusion}

In this paper, we revisited the optimization of DNNs under evolving curvature anisotropy. Our analysis suggested that a single fixed norm is generally insufficient to match the substantial geometric changes that occur during training. Motivated by this observation, we proposed a cosine $\ell_p$-norm scheme and incorporated it into SGD and SGDM, resulting in LPSGD and LPSGDM. We further established convergence guarantees for the proposed methods on nonconvex setting, showing that they attain an $O(T^{-1/2})$ convergence rate. Extensive experiments on CIFAR-10, CIFAR-100, and ImageNet-1K with VGG-11, ResNet-18, and ResNet-50 demonstrated that the proposed methods achieve consistently strong performance and competitive generalization. 

\section*{Acknowledgment}

This work was supported by grants from the National Natural Science Foundation of China under Grant 62302325. This work was also supported by the Natural Science Foundation of Jiangsu Province in China under Grant BK20230485 and by Project Funded by the Priority Academic Program Development of Jiangsu Higher Education Institutions.

\appendix
\section{HYPER-PARAMETERS}
\label{HYPER-PARAMETERS}
The selected hyper-parameters for all experiments are summarized in Tables \ref{tab:cifar10_hparams}--\ref{tab:imagenet_hparams}. Specifically, Table \ref{tab:cifar10_hparams} describes
the parameter settings of different optimizers on CIFAR-10, using VGG-11, ResNet-18, and ResNet-50. Table \ref{tab:cifar100_hparams} presents the parameter settings of different optimizers on CIFAR-100, using VGG-11, ResNet-18, and ResNet-50. Table \ref{tab:imagenet_hparams} reports the parameter settings of different optimizers on ImageNet-1K, using ResNet-50.

\begin{table*}[!t]
\centering
\caption{CIFAR-10 hyper-parameters.}
\label{tab:cifar10_hparams}
\setlength{\tabcolsep}{7pt}
\renewcommand{\arraystretch}{1.15}
\makebox[\textwidth][c]{%
\resizebox{1.25\textwidth}{!}{%
\begin{tabular}{l l c c c l c c c c}
\hline
Model & Optimizer & Batch Size & $p$ & Learning Rate & Schedule & $\beta_1$ & $\beta_2$ & $\lambda$ & $\epsilon$ \\
\hline
VGG-11    & SGD w/ Momentum & 128 & -- & 0.1    & warmup cosine decay & 0.9 & --    & 0.001  & 1e-8 \\
VGG-11    & AdamW           & 128 & -- & 0.001  & warmup cosine decay & 0.9 & 0.999 & 0.1    & 1e-8 \\
VGG-11    & Lion            & 128 & -- & 0.0001 & warmup cosine decay & 0.9 & 0.999 & 0.001  & --   \\
VGG-11    & LPSGDM          & 128 & 6  & 0.001  & warmup cosine decay & 0.9 & --    & 0.01   & 1e-8 \\
\hline
ResNet-18 & SGD w/ Momentum & 128 & -- & 0.1    & warmup cosine decay & 0.9 & --    & 0.001  & 1e-8 \\
ResNet-18 & AdamW           & 128 & -- & 0.001  & warmup cosine decay & 0.9 & 0.999 & 0.1    & 1e-8 \\
ResNet-18 & Lion            & 128 & -- & 0.0001 & warmup cosine decay & 0.9 & 0.999 & 0.001  & --   \\
ResNet-18 & LPSGDM          & 128 & 6  & 0.001  & warmup cosine decay & 0.9 & --    & 0.01   & 1e-8 \\
\hline
ResNet-50 & SGD w/ Momentum & 128 & -- & 0.05   & warmup cosine decay & 0.9 & --    & 0.0005 & 1e-8 \\
ResNet-50 & AdamW           & 128 & -- & 0.001  & warmup cosine decay & 0.9 & 0.999 & 0.0001 & 1e-8 \\
ResNet-50 & Lion            & 128 & -- & 0.0001 & warmup cosine decay & 0.9 & 0.999 & 0.0005 & --   \\
ResNet-50 & LPSGDM          & 128 & 6  & 0.002  & warmup cosine decay & 0.9 & --    & 0.0001 & 1e-8 \\
\hline
\end{tabular}
}
}
\end{table*}

\begin{table*}[!t]
\centering
\caption{CIFAR-100 hyper-parameters.}
\label{tab:cifar100_hparams}
\setlength{\tabcolsep}{7pt}
\renewcommand{\arraystretch}{1.15}
\makebox[\textwidth][c]{%
\resizebox{1.25\textwidth}{!}{%
\begin{tabular}{l l c c c l c c c c}
\hline
Model & Optimizer & Batch Size & $p$ & Learning Rate & Schedule & $\beta_1$ & $\beta_2$ & $\lambda$ & $\epsilon$ \\
\hline
VGG-11    & SGD w/ Momentum & 128 & -- & 0.1    & warmup cosine decay & 0.9 & --    & 0.001   & 1e-8 \\
VGG-11    & AdamW           & 128 & -- & 0.001  & warmup cosine decay & 0.9 & 0.999 & 0.001   & 1e-8 \\
VGG-11    & Lion            & 128 & -- & 0.0001 & warmup cosine decay & 0.9 & 0.999 & 0.0001  & --   \\
VGG-11    & LPSGDM          & 128 & 8  & 0.0005 & warmup cosine decay & 0.9 & --    & 0.0005  & 1e-8 \\
\hline
ResNet-18 & SGD w/ Momentum & 128 & -- & 0.1    & warmup cosine decay & 0.9 & --    & 0.01    & 1e-8 \\
ResNet-18 & AdamW           & 128 & -- & 0.01   & warmup cosine decay & 0.9 & 0.999 & 0.1     & 1e-8 \\
ResNet-18 & Lion            & 128 & -- & 0.001  & warmup cosine decay & 0.9 & 0.999 & 0.1     & --   \\
ResNet-18 & LPSGDM          & 128 & 6  & 0.01   & warmup cosine decay & 0.9 & --    & 0.1     & 1e-8 \\
\hline
ResNet-50 & SGD w/ Momentum & 128 & -- & 0.1    & warmup cosine decay & 0.9 & --    & 1e-5    & 1e-8 \\
ResNet-50 & AdamW           & 128 & -- & 0.001  & warmup cosine decay & 0.9 & 0.999 & 0.1     & 1e-8 \\
ResNet-50 & Lion            & 128 & -- & 0.001  & warmup cosine decay & 0.9 & 0.999 & 1e-5    & --   \\
ResNet-50 & LPSGDM          & 128 & 6  & 0.001  & warmup cosine decay & 0.9 & --    & 0.1     & 1e-8 \\
\hline
\end{tabular}
}
}
\end{table*}

\begin{table*}[!t]
\centering
\caption{ImageNet-1K hyper-parameters.}
\label{tab:imagenet_hparams}
\setlength{\tabcolsep}{7pt}
\renewcommand{\arraystretch}{1.15}
\makebox[\textwidth][c]{%
\resizebox{1.25\textwidth}{!}{%
\begin{tabular}{l l c c c l c c c c}
\hline
Model & Optimizer & Batch Size & $p$ & Learning Rate & Schedule & $\beta_1$ & $\beta_2$ & $\lambda$ & $\epsilon$ \\
\hline
ResNet-50 & SGD w/ Momentum & 128 & -- & 0.01   & warmup cosine decay & 0.9 & --    & 0.0005 & --   \\
ResNet-50 & AdamW           & 128 & -- & 0.001  & warmup cosine decay & 0.9 & 0.999 & 0.0001 & 1e-8 \\
ResNet-50 & Lion            & 128 & -- & 0.0003 & warmup cosine decay & 0.9 & 0.99  & 0.01   & --   \\
ResNet-50 & LPSGDM          & 128 & 9  & 0.002  & warmup cosine decay & 0.9 & --    & 0.005  & 1e-8 \\
\hline
\end{tabular}
}
}
\end{table*}

\end{document}